\documentclass{article}

\usepackage{times}
\usepackage{amsmath}
\usepackage{enumitem}
\usepackage{amssymb}
\usepackage{algorithmic}
\usepackage{algorithm}
\usepackage{fullpage}
\usepackage{natbib}

\usepackage{amsthm}
\usepackage{amsfonts}
\usepackage{comment}
\usepackage{graphicx}
\usepackage{hyperref}
\usepackage{float} 
\usepackage{color}

\include{defs}
\def\shownotes{0}  
\ifnum\shownotes=1
\newcommand{\authnote}[2]{{$\ll$\textsf{\footnotesize #1 notes: #2}$\gg$}}
\else
\newcommand{\authnote}[2]{}
\fi

\newtheorem{thm}{Theorem}[section]
\newtheorem{lem}{Lemma}[section]
\newtheorem{defn}{Definition}[section]
\newtheorem{prop}{Proposition}[section]

\newcommand{\E}{\mathbb{E}}

\begin{document}
\title{Provable benefits of representation learning}
\author{Sanjeev Arora \and Andrej Risteski \thanks{Princeton University, \texttt{arora, risteski@cs.princeton.edu}. Supported by NSF grants CCF- 1302518, CCF-1527371,  Simons Investigator Award, Simons Collaboration Grant, and ONR- N00014-16-1-2329}}
\maketitle

\vspace{0cm}
\begin{abstract}

There is general consensus that learning representations is useful for a variety of reasons, e.g. efficient use of labeled data (semi-supervised learning), transfer learning and understanding hidden structure of data.
Popular techniques for representation learning include clustering, manifold learning,  kernel-learning, autoencoders, Boltzmann machines, etc. 

To study the relative merits of these techniques, it's essential to formalize the definition and goals of representation learning, so that they are all become instances of the same definition. This paper introduces such a formal framework that also formalizes the {\em utility} of learning the representation. It is related to previous Bayesian notions, but with some new twists. We show the usefulness of our framework by exhibiting simple and natural settings -- linear mixture models and loglinear models, where the power of representation learning can be formally shown. In these examples, representation learning can be performed provably and efficiently under plausible assumptions (despite being NP-hard), and furthermore: (i) it greatly reduces the need for labeled data (semi-supervised learning) and (ii) it allows solving classification tasks when simpler approaches like nearest neighbors require too much data (iii) it is more powerful than manifold learning methods.



\end{abstract}


\newcommand{\mX}{{\mathcal X}}
\newcommand{\mH}{{\mathcal H}}
\newcommand{\mC}{{\mathcal C}}
\newcommand{\from}{\colon}

\section{Introduction}
\label{s:intro}

Learning good {\em representations} of the data is generally considered a good idea,  for reasons such as semisupervised learning (using plentiful unlabeled data to extend the utility  of small amounts of labeled data),   {\em transfer learning} (learning structure in data from one domain and applying to another domain). But there is no single definition of what it means to learn a representation. Instead, representation learning is an umbrella term for a collection of techniques: clustering, manifold learning, and kernel-learning, as well as notions in deep learning such as autoencoders and Boltzmann machines.  Even a multilayer neural network trained discriminatively on classification tasks can be viewed as first learning a representation of data, which the top layer uses for classification.  
In Bayesian formulations of learning, representation learning is often viewed as finding MAP ({\em maximum a posteriori}) estimates of some latent variables or training a generative model for the input data. 

Formalizing the definition and goals of representation learning is important in order to start an exploration of the relative merits of various approaches.  This paper attempts to start such a study by giving a definition. We then study the notion in two natural settings--- \emph{linear mixture models} and \emph{loglinear models}. These are more complex than clustering, but not as complex as multilayer neural networks, so it is still possible to study algorithms that are provably efficient. 

Importantly, in these settings we can not only design polynomial-time algorithms that learn good representations (under plausible assumptions), but also exhibit a \emph{provable separation} in terms of time and sample complexity between our best polynomial-time algorithm and simpler methods. For instance, unlabeled datapoints can improve supervised classification (Section~\ref{subsec:semisupervised}), and be substantially more efficient than nearest neighbor techniques (Section~\ref{subsec:NN}), or more generally than methods for manifold learning that rely only on nearest neighbor information on the datapoints (Section~\ref{subsec:manifold}).
In other words learning a representation is much more efficient than relying upon local metric structure on original datapoints.


\section{A Framework for Representation Learning}
\label{sec:framework}

Our framework is informed by the following motivations and principles.

First, representation learning corresponds to mapping from a \textquotedblleft low level \textquotedblright\ description to a \textquotedblleft high level\textquotedblright description.  Such a map will necessarily be many-to-one: very different data points may map to the same or nearby representations, e.g. a multitude of images of different people ( totally different at pixel level) may correspond to the high-level representation expressing that it contains a human face with glasses.  


Second, there is an implicit description of which datapoints $x$ can correspond to a given representation $h$. Given that one $h$ can correspond to multiple $x$'s, it is natural to think of the datapoint $x$ being generated by a probabilistic process that uses its latent representation $h$. 

Third, we formalize the {\em utility} from learning the representation. The simplest utility could be to enable semi-supervised learning (i.e., use unlabeled data to learn the representation function and having done that, use small amounts of labeled data to do classification). A natural way to ensure this is by requiring the high-level \textquotedblleft similarity\textquotedblright of datapoints $x, x'$  correspond to closeness of the representations $h, h'$ in a natural metric (even though $x, x'$ themselves may be very different in every natural metric) -- thus classification in the representation space can be performed by a much simpler classifiers, e.g. SVMs or nearest neighbor classifiers. 

Thus we are led to the following framework.

\medskip

\noindent{\sc framework} \\
{\em 
\emph{Generative process}:  a datapoint $x$ is assumed to be probabilistically produced from its representation
 $h$ according to $x = g(h, r)$ where $g$ is a deterministic function and $r$ is a random seed chosen from some distribution. 
 
A \emph{representation function} $f$ is a mapping from the space $\mX$ of datapoints to a space 
$\mH$ of {\em representations} that inverts the above map.
 
Learning the representation function $f$ is {\em useful} for solving a set of classification problems $\mC$,
where every classifier $C \in \mC$ maps $\mH$ to a space of labels. Notice, every $C \in \mC$ can also be thought of as a map from
$\mX$ to a space of labels, which maps datapoint $x$ to the label  $C(f(x))$. Thus learning $f$ can help classify the original datapoints.}
 
\medskip
We will usually be interested in classifiers that produce binary outputs, though sometimes we will think of $C(h)$ as real-valued and lying in $[0,1]$, which can also be interpreted as a {\em probability} for outputting label $1$. If the set of labels is different from $\{0,1\}$ then $C(h)$ is a vector with real coordinates that sum to $1$, and the $i$th coordinate represents the probability of outputting label $i$.


The following definition  formalizes the goal of learning the representation function $f$ as finding an inverse of $g$, in other words to be able to produce 
 $h$ (approximately) given $x$. In the definition, constants $\gamma, \beta$ should be thought of as close to $1$. 

\begin{defn}[Encoder] Given the setting of the previous paragraphs, a function $f \from \mX \to \mH$ is a $(\gamma, \beta)$-{\em valid encoder} for the generative model $g$ if given $x =g(h, r)$ with probability at least $\beta$, 
$\|f(x) - h\| \leq (1-\gamma) \| h\|$. Here $\|\cdot\|$ is any suitable norm. 
\label{d:encoder}
\end{defn}

Of course, such a function may not exist for all generative models $g$ -- but we show interesting cases later where it does.  The following trivial observation quantifies the utility of such an encoder 

\begin{prop}[Trivial] If a classifier $C$ is $\alpha$-lipschitz, meaning $\|C(h) - C(h')\|_{\infty} \leq \alpha \|h - h'\|$ and an encoder $f$ is  $(\gamma, \beta)$-valid  then given a datapoint
$x = g(h, r)$ the probability is at least $\beta$ that $\|C(f(x)) - C(h)\|_{\infty} \leq (1-\gamma)\alpha\|h\|$.
\label{t:lip}
\end{prop}
In particular, if the labels were actually produced using some Lipschitz classifier $C$ then  learning the representation function allows us to effectively replace $x$ by its representation $f(x)$ for purposes of classification.


{\em Note:} While the term \textquotedblleft encoder\textquotedblright\ is borrowed from deep learning, the framework  is closely inspired by  classical {\em unsupervised learning,} which is also used in deep learning in models such as {\em variational autoencoders} and {\em Boltzmann Machines}. Section~\ref{subsec:bayesian} gives precise comparison with Bayesian notions. 

\subsection{Relation to other notions of representation learning}

Classical notions of representation learning fit naturally in  our framework.

%
{\bf Clustering:} The latent representation of a point is the cluster it belongs to (the form of $g$ would determine the shape of the clusters.) In accordance with our notion of utility, clustering is often used for semi-supervised classification settings, when the labels are believed to change little within a cluster. (\cite{seeger2000learning})

 
{\bf Manifold learning} assumes data lies on a low-dimensional manifold ${\cal M}$; the datapoint is assumed to be generated from the manifold, possibly with added noise. It is used for semisupervised classification using the \emph{manifold assumption} when labels are believed to be Lipschitz with respect to manifold distance, rather than ambient distance. (\cite{belkin2004semi}) This fits in our framework: one can set $g$ to be some parametrization of the manifold ($g$ only would depend on $h$, not on $r$).  Manifold learning methods typically run in time \emph{exponential} in the dimension of the manifold since they effectively construct an $\epsilon$-net (\cite{saul2003think, tenenbaum2000global}). In fact, in Section \ref{subsec:manifold} we will even show a simple and natural settings where these methods are \emph{provably weaker} than a representation algorithm we design.  



{\bf Kernel learning} learns a kernel from data --often using nonconvex optimization---and uses it for classification~\cite{corteslearning}. Our encoder function finds an \emph{embedding} of the data -- hence induces an obvious kernel. (Interestingly, this learnt embedding is nonlinear in some settings in this paper.)
	

\subsubsection{Related notions}

{\bf Nearest neighbor classification} \emph{doesn't} use an alternate representation of the input -- and in its most basic form prescribes using the label of the majority of the $k$ nearest neighbors in the training set. Generalization properties of this  nonparametric method are analysed e.g. in \cite{chaudhuri2014rates}. The implicit assumption however is that similarity for purposes of classification is captured by some metric distance among the original datapoints, the labels are Lipschitz with respect to this distance, and that that datapoints are dense in the space, so the neighbors of a new point are relevant to deciding its label. As we show in Section~\ref{sec:necessary}, such assumptions are violated in common, natural settings -- thus nearest neighbors can often be \emph{provably weaker} than the representation learning algorithm we describe.

{\bf Locality sensitive hashing} is another approach useful for speeding up similarity queries in high dimensions, by producing a low-dimensional representation of the data which approximately preserves pairwise distances. But as the sentence similarity example showed, we are interested in settings where there isn't an apparent metric structure to preserve in the given data, and these ideas don't apply (at least if the number of datapoints is reasonable as a function of the natural parameters).
Additionally, representation learning by fitting a Restricted Boltzmann Machine for instance has been shown to improve both computation time and accuracy of locality sensitive hashing in structured settings like textual data. (\cite{salakhutdinov2007semantic})


Finally, we compare to the recent paper of \cite{hazan2016non} which also formalizes the utility of representation learning using its help in classification tasks. Unlike our generative model approach, they have an "assumption-free" approach to representation learning, so that the utility of the representation is that it should contain enough information to learn \emph{every} classifier. One quickly realizes that such a representation should preserve \emph{all} information in the input, which leads to viewing representation learning as {\bf lossless compression}. The problem with this identification of representation learning with lossless compression is that many forms of compression (e.g Lempel-Ziv) do not lead to very good representations in practice. 
In fact, in Section \ref{s:unreliable} we describe a setting where the correct representation is a \emph{not} a lossless compression, but one which throws away a ``noisy'' portion of the input. Thus, the results of our paper are very distinct from \cite{hazan2016non}.

\subsection{Clarification of connections with Bayesian modeling}
\label{subsec:bayesian}
We further explain the similarity and differences with the usual Bayesian formulations: it is common for unsupervised learning to be formalized using a latent variable generative model, namely, a distribution $p_{\theta}(x, h)$ where $x$ is the {\em visible}  data, $h$ is the {\em unobserved} part, and $\theta$ is a vector of model parameters. However, the standard goal in unsupervised learning is phrased in a different way: one must learn $\theta$ that maximizes the {\em log probability} of the training data. The learning is said to {\em generalize} if this score is comparable for training data and held-out data. Having learnt such a $\theta$, a candidate representation for $x$ is the  value of $h$ that maximizes $p_{\theta}(h|x)$.
Our framework is a special subcase whereby $p_{\theta}(x | h)$ is just 
$\Pr_r[g(h, r) =x]$, and $\theta$ is the \textquotedblleft code\textquotedblright for $g$ (e.g., a description of the circuit computing $g$). The new twist is the different goal for the algorithm: instead of focusing on measures of distributional fit such as log probability, our goal is to learn a $(\gamma, \beta)$-valid encoder, so that we can better classify the datapoints. (Aside: In the settings studied below, the encoder is either linear or a 1-layer neural net.)

From a Bayesian viewpoint, one concern is that a datapoint $x$ could be generated in multiple ways, e.g., $x= g(h, r)$ and  $x= g(h', r')$ -- but we require outputting a single $h$. 
One may imagine the necessary condition for the encoder to exist is that any such $h, h'$ are fairly close to each other. But this isn't quite correct, since the encoder can afford to ignore any set of $h'$ whose total chance of producing $x$ is small -- the correct condition is identified in the following simple lemma which was implicit in the calculations of \cite{arora2016provable}. It clarifies the encoder notion and properly situates it in the Bayesian world.


\begin{prop} If a $(1-\delta^2, \gamma)$-encoder exists for the generative model $g$ then with probability at least $1-\delta$ over the choice of $x$, the posterior distribution $\Pr_{h|x}$ satisfies $\Pr_{h|x} [\|h - f(x)\| \leq (1-\gamma)\|h\|] \geq 1-\delta$. 
\label{l:concentration_posterior}
\end{prop}
\begin{proof} 
The proof was implicit in calculations in \cite{arora2016provable}.
Suppose a $(\beta, \gamma)$-encoder $f$ exists. Then, by the definition of an encoder, 
$$\Pr_{x | h} \left[ \|f(x) - h\| \leq (1-\gamma)\|h\| \right] \geq 1-\delta^2$$
Clearly, this implies 
\begin{align*} \E_{x} \left[ \Pr_{h | x} \left[ \|f(x) - h\| \leq (1-\gamma)\|h\| \right] \right] &= \Pr_{x, h} \left[ \|f(x) - h\| \leq (1-\gamma)\|h\| \right] \\
&= \E_{h} \left[ \Pr_{x | h} \left[ \|f(x) - h\| \leq (1-\gamma)\|h\| \right] \right] \geq 1-\delta^2 \end{align*}

By Markov, this implies with probability at least $1-\delta$ over the choice of $x$, 
$$\Pr_{h | x} \left[ \|f(x) - h\| \leq (1-\gamma)\|h\| \right] \geq 1-\delta$$
as we needed. 

\end{proof}

\section{Generative models}
\label{s:basicmodels}

For simplicity, the running example in this paper will be {\em collaborative filtering}: predicting whether a user will like a movie given his/her past ratings, the so-called {\em Netflix} problem. The same reasoning applies to many other settings, including documents in natural languages. We consider two two basic families of generative models: mixture models (which are linear) and loglinear models. We focus on these, as they are substantially more complicated than clustering, but still simple enough that we can design provable algorithms for representation learning.   

\subsection{Linear models} 

The prototypical linear generative model are topic models, frequently used to understand latent structure in text. The writer is assumed to have a set of $k$ topics in mind, each of which is a distribution on $M$ words. Thus a topic $A_i$ is an $M$-dimensional vector of word frequencies. These are the model parameters $\theta$. 

When producing a document $x$, the generative model stipulates the writer randomly draw a subset of $s$ topics out of $k$, where topic $A_i$ has proportion $h_i$. Thus only 
$s$ of the $h_i$'s are nonzero and $\sum_i h_i =1$. Then $\sum_i h_i A_i$ is a distribution on words and the document $x$ is assumed to be a sequence of iid draws from this distribution. Note that $x \in {\mathbf Z}^N $ is the so-called \textquotedblleft bag of words representation,\textquotedblright where the $j$th coordinate is the number of times word $j$ occurs in the document.

Learning topic models involves learning the topic vectors given a sample of documents, and {\em inference} consists of inferring the most likely topic proportions $h$ given a document $x$. Techniques such as MCMC and variational inference are frequently used, but they come with no provable guarantees. In fact,~\cite{sontag2011complexity} show that the inference problem can even be \#P-hard with no assumptions on the topic structures. Recently linear algebraic optimization methods with provable guarantees have been developed under realistic assumptions~\cite{arora2016provable}. Note, by the discussion in Section~\ref{subsec:bayesian} we are only interested in settings where this $h$ is almost uniquely defined.

Our framework introduced above posits that the classification tasks we are interested in are those where the label of a document $x = g(h, r)$ is $C(h)$ where $h$ was the vector of topic proportions used to generate the document. 

Though the usual usage is textual data, in the rest of the paper we will use collaborative filtering terminology: we have $M$ movies, $k$ \emph{genres}, each of which is a subset of movies.\footnote{In the topic models case, this would correspond to a uniform distribution over the words in this subset; we assume this for simplicity of exposition only.} For simplicity, we assume each genre has $m$ movies. Each user likes $s$ genres (the latent variables). A user's affinity to a genre is given by a weighting, and the overall liking of a movie is proportional to the sum of the user's weights for all the genres that this movie is contained in. $T$ movies are then independently generated, by picking a movie proportional to his/her liking for it.



\subsection{Log-linear models}

In a log-linear model --- of which there are many variants ---the user again has latent affinities for genres, 
but the probability distribution specifying how to generate rated movies has a \emph{log-linear} form. 
Log-linear models are extensively used in many applications, including images (\cite{salakhutdinov2009deep}), speech (\cite{mohamed2012acoustic}), text (\cite{mikolov2013distributed}), recommendation systems (\cite{salakhutdinov2007restricted}). The most classical form is RBM (Restricted Boltzmann Machine), in which the joint probability of the latent and observable variables follows a distribution $p(x,h) \propto \exp(x^T W h)$, for some matrix $W$; $x$ and $h$ can be either binary or continuous (the latter case usually is called a \emph{Gaussian} RBM). 

The primary motivation behind the RBM form is that the inference problem of calculating $p(h|x)$ is trivial, as $p(h|x)$ is a product distribution. Unfortunately, learning the parameters of the matrix $W$ from data requires solving a highly non-convex optimization problem; in fact, even evaluating the gradients during the optimization is difficult as it requires computing the \emph{partition function}: $\sum_{x,h} \exp(x^T W h)$. \footnote{In the continuous case, the sum becomes an integral.} Indeed, \cite{long2010restricted} show that calculating this quantity even to within an exponential multiplicative factor is NP-hard. 

Our model is a slight modification of the above: we have $p(x,h) = p(h) p(x|h)$, $p(h)$ is a very simple distribution (uniform over the sphere), while $p(x|h)$ essentially has the same exponential form as the RBM. Note that in this model, a-priori both inference and parameter learning seem hard: terms like $p(x|h)$ involve a partition function $Z_h$ corresponding to each latent $h$ and make the objectives non-convex. 

\section{Efficient and provable representation learning}

In this section, we show that in the models introduced above, it is possible to learn representation provably and efficiently, under realistic assumptions on the model parameters. In the linear case, this was already implicit in prior work, so we merely place it in our framework. In the log-linear case, no such results were known, to the best of our knowledge.

\subsection{Representation learning for linear models}
\label{subs:linear}

Linear models have received a lot of attention in recent years and model parameters can be learnt in polynomial time under realistic assumptions~(\cite{anandkumar2012spectral, anandkumar2014tensor, arora2012learning}). There had been less work on provably learning encoders, but recently there has been progress on that as well -- though not in the terminology and framing used in this paper. Hence, we will merely rephrase these results in our language.

The particular simplified model for collaborative filtering we described is a special case of the Latent Dirichlet Allocation (LDA) model, which in general would allow the probabilities of emitting movies from a genre to be different, as well as the probabilities of the different genre that a given user is fond of to vary. The work of \cite{anandkumar2014tensor} gives a provable tensor-based algorithm for learning the parameters of the model when the $M \times k$ movie-genre matrix, whose entries consist of the probabilities of emitting a particular movie, given that a particular genre was selected, has full (column) rank. 

However, by Lemma~\ref{l:concentration_posterior}, our Definition \ref{d:encoder} of an encoder requires not only that we are able to perform \emph{inference} (calculating the values of the latent variables, given the observable ones), but also that the posterior distribution on the latent variables \emph{concentrates} on a particular value. This question has been studied in \cite{arora2016provable}, who prove that under a certain linear-algebraic condition on the movie-genre matrix, inference can be performed. More precisely, they consider movie-genre matrices $A$ of small $l_{\infty} \to l_1$-\emph{condition number}: i.e. matrices $A$ that satisfy $\frac{\|Av\|_1}{\|v\|_{\infty}} \geq \frac{1}{\lambda}$ -- a weakening of the $\l_1$ condition number: the minimum $\lambda$, s.t. $\frac{\|Av\|_1}{\|v\|_{1}} \geq \frac{1}{\lambda}$, originally used by \cite{kleinberg2004using}. (Intuitively, this condition says that the observables in expectation are different for differing latent vectors: namely it implies $\|Ah_1 - Ah_2\|_1 \geq \frac{1}{\lambda} \|h_1 - h_2\|_1$, and $Ah$ is the expectation of the observables for linear models.) The the condition was empirically verified, and found to be satisfied on topic models learned from real-life corpora like New York Times. 

They show this condition is equivalent to the existence of a \emph{low-variance pseudo-inverse} matrix to $A$: a matrix $B$, s.t. $BA = I$, $|B|_{\infty} \leq \lambda$. We will recast their inference theorem in the language of this paper. For notational convenience, for a scalar $\tau$, denote by $\phi_{\tau}: \mathbb{R}^k \to \mathbb{R}^k$ the function, s.t. $(\phi_{\tau}(z))_i = z_i$, if $z_i \geq \tau$, and $0$ otherwise. They show: 

\begin{thm}[\cite{arora2016provable}] Suppose the number of movies rated satisfies $T = \Omega(\lambda^2 s^2 \log M)$. Then, if $x$ is the bag-of-words representation of the document, the function $f(x) = \phi_{\tau}(\frac{1}{T} B x)$, for $\tau = 2 \lambda \sqrt{\frac{\log G}{T}}$ is a \\ $\left(\lambda \sqrt{\frac{s}{T}},1-\exp(-\log^2M)\right)$-valid encoder.\footnote{The suitable norm for the encoder here is 1-norm.}   
\end{thm}

The crucial point to note is that $T$ depends only logarithmically on the total number of movies. 
This is the key reason why representation learning can be (exponentially) more powerful than popular techniques such as nearest neighbors, which due to birthday paradox calculations will need $T > \sqrt{M}$ -- as we will prove later in Section~\ref{sec:necessary}. 

\subsection{Representation learning for log-linear models}
\label{s:loglinear}


In this section, we will exhibit encoders for a particular kind of log-linear models. These models have received a lot of attention in various application domains including text analysis and recommender systems. (The final solution to the netflix problem used an ensemble of various models, including linear and loglinear.) There are many kinds of log-linear models -- here we use a variant using continuous variables that is closer to more recent models for text analysis that use a continuous representation for text, Word2Vec, (\cite{mikolov2013distributed}), one of the most popular algorithms for deriving word embeddings, which in turn,  has been almost directly used (\cite{ozsoy2016word}) for recommendation systems. This loglinear model can also be seen as one layer of deep belief nets.

Sticking with the movie recommendation setting, we assume as before that there are $M$ movies, and a vector $W_x \in \mathbb{R}^d$ assigned to each movie $x$. A user is represented by a latent vector that is  a real-valued vector $h \in \mathbb{R}^d$, distributed as a uniform vector over the unit sphere, and the observable vector $(x_1, x_2, \dots x_T) \in \{0,M\}^T$ will be generated by independently choosing $T$ movies according to the loglinear distribution $p(x | h) \propto \exp(\langle W_x, h \rangle)$. Thus $h$ is analogous to the ``mix'' of the genres of a particular user in the mixture model; the model then tends to emit the movies whose movie vectors are most correlated with $h$. But the preference for individual movies has a different functional form.

Learning such models or doing inference (i.e., compute the encoder function) can be NP-hard in such settings.
Our analysis for the log-linear models will use simplifying assumptions introduced in  \cite{arora2016latent}, whose goal was to explain interesting properties of word embeddings.  
The key step is to make a distributional assumption on the movie vectors $W_x \in \mathbb{R}^d$: we will assume they are generated as $W_x = B · v$, where $v$ is sampled from a spherical Gaussian distribution $N(0, I)$, and $B$ is a scalar s.t. $B = O(1)$.  Thus movie vectors are uniformly distributed in space. 

Under these assumptions, learning the vectors $W_x$ can be empirically done by solving a non-convex optimization problem involving a weighted low-rank matrix approximation, by Theorem 2.2 in \cite{arora2016latent}. The paper~\cite{arora2016latent} gives empirical evidence that gradient descent convergence to a good solution to the objective in question. \footnote{A subsequent paper~\cite{li2016recovery} gives an analysis of weighted matrix factorization in an average-case setting, but it cannot yet show convergence under the conditions studied in~\cite{arora2016latent}.} The contribution of our paper is to show existence of a good encoder function given the vectors $W_x$ which, as mentioned, can be obtained at least in practice.



\begin{thm} Under the setup of this section, if  $T =\Omega(\log M)$ and the dimension satisfies $d = \Omega(B^2 \log^2 M )$, $M = \Omega(d^2/B^4)$, with high probability over the choice of movie vectors $W_x$, 
the function $f\left((x_1, \dots, x_T)\right) = \frac{\sum_{i=1}^T W_{x_i}}{\|\sum_{i=1}^T W_{x_i}\|} $ is a $(1-o(1), 1 - \exp(-\log^2 M))$-valid encoder. 
\label{t:wordembed}
\end{thm}

The proof of this theorem is fairly technical, involving non-trivial concentration bounds for certain functions of Gaussian variables. 
We will provide here a high-level sketch of the proof, followed by a formal proof in Appendix~\ref{s:missingword}. We will in fact show that with high probability over the choice of movie vectors $W_x$, and the observable $(x_1, x_2, \dots, x_T) \in \{0,M\}^T$, 

$$\bigg\langle \frac{\sum_{i=1}^T W_{x_i}}{\|\sum_{i=1}^T W_{x_i}\|}, h \bigg\rangle \geq 1-o(1) $$

which clearly implies the theorem statement. 

Because of the distributional assumption on the movie vectors and latent vectors $h$, we may without loss of generality assume $h= e_1$.  
We will show that with high probability under the choice of movie vectors and the distribution $P(x|h)$, when the parameters satisfy the conditions in the theorem statement, we have:
\begin{equation} \sum_{i=1}^T (W_{x_i})_1 \geq (1-o(1)) T \frac{B^2}{4d} \label{eq:denom_main} \end{equation}
Moreover, we will show that this implies that with high probability,
\begin{equation} \left\|\sum_i W_{x_i} \right\|  \leq (1+o(1)) \left(\sum_{i=1}^T (W_{x_i})_1\right)  \label{eq:num_main} \end{equation}
which will imply the statement of the lemma. 

The intuition for why this is true is as follows. For both inequalities, we need to consider quantities like 
$$\E_{p(x|h)}[(W_x)_j] $$
where $\E_{p(x|h)}[\cdot]$ is the expectation with respect to the conditional distribution on $h$. (Note this itself is a random quantity due to the prior on the movie vectors.)

Since $h=e_1$, $p(x|h)$ prefers to generate movies that have a large first coordinate. Quantifying this will give the first bound \eqref{eq:denom_main}, and it requires calculating a Gaussian integral (since the movie vectors have Gaussian coordinates) -- but it's additionally complicated by the fact that we have to condition on the event $\mathcal{F}$ that the partition function concentrates. This is non-trivial since then we are performing a Gaussian integral over only a subset of $\mathbb{R}$ (which correspond to the event $\mathcal{F}$ happening), and since $(W_x)_j$ can take both positive and negative values, this must be done carefully. 

On the other hand, to show the second bound \eqref{eq:num}, note that
$$\left\|\sum_i W_{x_i} \right\| = \sqrt{\sum_{j=1}^d \left(\sum_{i=1}^T (W_{x_i})_j\right)^2}$$
From the first inequality, we know that with high probability $\displaystyle \left(\sum_{i=1}^T (W_{x_i})_1\right)^2 \geq (1 - o(1)) \left(T \frac{B^2}{4d}\right)^2$. We will show this dominates $\displaystyle \sum_{j=2}^d \left(\sum_{i=1}^T (W_{x_i})_j\right)^2$, which suffices to prove the lemma. More precisely, we will show that for any $j \neq 1$,  $\displaystyle \left(\sum_{i=1}^T (W_{x_i})_j\right)^2 \leq (1+o(1)) T^2/(Md)$. The intuitive reason this happens is that the distribution $p(x|h)$ doesn't depend on coordinates $j \neq 1$. Hence, if we consider
$$\E_{p(x|h)}[(W_x)_j] = \sum_{x \in M} \frac{\exp((W_x)_1)}{Z} (W_x)_j$$  
the RHS behaves roughly like a sum of M i.i.d Gaussians with variance roughly $1/d$, which gives the desired bound. This will imply the claim if $M =\Omega( d^2/B^4 )$.



\section{The power of representation learning}
\label{sec:necessary}
This section illustrates the power of representation learning in the simple settings we introduced. In particular, we show: \\(i) efficient algorithms exist for learning the encoder from unlabeled data, but no obvious metric structure can be detected in the vanilla representation. Hence, methods like nearest-neighbors and manifold learning will do poorly. Thus, we exhibit a separation from simpler, popular approaches. \\(ii) it can greatly reduce the number of \emph{labeled samples} required -- thus helps for semi-supervised learning settings. 


Note that since the goal is a {\em separation} between different definitions and algorithms for learning, our result holding in very simple settings is a feature, rather than a bug. 
 
\subsection{Separation from nearest-neighbor classifiers}
\label{subsec:NN}

\noindent {\bf Setup:} Recall that the setting is predicting movies that a user will like given her ratings for past movies. The total number of movies is $M$. For simplicity, we assume the rating is binary, and in fact the user has only given a list of $T$ movies she has liked in the past. 

A simple generative model for the user's ratings is the following.  There are $M$ movies, and
there are $k$ {\em genres}, each of which is a subset of movies. Each genre we assume has $m$ movies.

The user's latent vector is a {\em subset} of $s$ genres out of $G$: she likes movies in these $s$ genres
and no others. Thus her latent vector satisfies $h \in \{0,1\}^G, |h|_0 = s$. This latent vector leads to the following observed behavior: she outputs a random subset $x \in \{0,1\}^T, \|x\|_0 = T$ 
of $T$ movies drawn at random from the union of the genres she is fond of. Note this is slightly different than drawing $T$ movies independently: for collaborative filtering no repetitions among movies makes more sense, and we exhibit our counterexamples in this setting. In fact, this setting is slightly harder because of dependencies between the movie emissions -- and for completeness, we re-prove our lower bounds in the independent emission setting as well in Appendix~\ref{s:independent}. 

We will see below that there are two regimes: when $T \ll \sqrt{m}$  or $T \gg \sqrt{m}$. As one would expect, the basic intuition is similar to the {\em birthday paradox}: in order to get informative and substantial overlaps between the users when they share genres they like, we need to see at least $\sqrt{m}$ ratings. 

\subsubsection{$T \ll \sqrt{m}$: nearest neighbors are unreliable}
\label{s:unreliable}

We will first show that when $T < m^{0.5-\epsilon}$ for any constant $\epsilon >0$ then nearest neighbor methods
run into trouble. In fact this is even true when $s=1$, that is, each user is fond of only one genre, and the genre structure is exceedingly simple: 
all the genres share some subset of ``common'' movies $R$, s.t. $|R| = pm$ for a constant $0 < p < 1$; additionally, the remaining $(1-p)m$ movies in each genre are unique for that genre, namely they only appear in that genre and no other. Let the number of observed ratings be $T = O(m^{1/2-\epsilon})$. 
In this case, an encoder can be learnt easily from polynomial number of samples as a simple subcase of the method in Section~\ref{subs:linear}: indeed, it's quite easy to see the existence of the movies unique to the genres implies the $l_1$ condition number of the movie-genre matrix is small. Using an encoder, the user's latent genre id can be recovered with high probability once she has rated $\frac{1}{1-p} \log m$ moves. By contrast, we have:

\noindent{\em Nearest neighbors are unreliable:} consider the task of predicting whether the users like some movie $i$. This can be cast as predicting a label $l_i(h)$ for each user, s.t. 
$l_i(h) = 1$, if $\exists j \in [k]$ , s.t. $h_j \neq 0$ and $i$ is in genre $j$. 
We show that a prediction based on nearest-neighbors on the observables $x$ is bound to be wrong on a very large fraction of the users. Namely, we show: \\(i) with probability at least  $1 - 1/\mbox{poly(m)}$ the set of movies output by ever pair of users overlaps in at most $O(1)$ movies; \\(ii) for any constant $\tau$, with $1-1/\mbox{poly(m)}$ probability, the fraction of users that share at least $\tau$ number of rated movies, but like a different genre is a factor of $k$ larger than the fraction of users that like the same genre.  

For notational convenience, for a user $u$, let $\chi(u)$ denote the genre that the user is fond of.

\begin{thm} In the setup of this section,\\
(i) For any constant $c$, if the number of users is at most $m^c$, with probability at least $1 - \frac{1}{m^c}$, the number of ratings two users share is at most $3c/\epsilon$. \\
(ii) For any $\tau \leq \frac{3c}{\epsilon}$, 
    $$\Pr[\chi(u_1) \neq \chi(u_2) | u_1, u_2 \mbox{ share at least } \tau \mbox{ ratings}] \gtrsim k \Pr[\chi(u_1) = \chi(u_2) | u_1, u_2 \mbox{ share at least } \tau \mbox{ ratings}]$$  
\label{t:toofew}
\end{thm} 
Note the above theorem in a very strong way shows that nearest neighbors -- even interpreted very generously -- cannot work: namely for any reasonable threshold $\tau$ of shared ratings (i.e. one that will result in overlaps being observed with a reasonable number of users), it is substantially more likely for two users liking different genres to have that many shared ratings, as compared to two users liking the same genre. The proof of this theorem is based on calculations involving careful approximations of binomials (which appear due to the nature of the generative model).

\begin{proof}[Proof of Theorem~\ref{t:toofew}]
Let us prove (i) first. 

Let $u_1$, $u_2$ be two randomly chosen users and let $\chi(u_1), \chi(u_2)$ be their respective genres. Let $\mathcal{O}$ be a random variable denoting the size of the overlap of the ratings of $u_1$ and $u_2$. Then, we have
\begin{equation} \Pr[\mathcal{O} = \tau | \chi(u_1) = \chi(u_2)] = \frac{\binom{m-T}{T-\tau} \binom{T}{\tau}}{\binom{m}{T}}  \label{eq:overlapsame}\end{equation}
Indeed, regardless of what the set of movies rated by $u_1$ is, the there are $\binom{m-T}{T-\tau} \binom{T}{\tau}$ subsets of size $T$ that intersect with it in $\tau$ movies. 

For $b = o(\sqrt{a})$, note that we have $\frac{a^b}{b!} \geq \binom{a}{b} \geq (1-o(1))\frac{a^b}{b!}$. 

Hence, 
\begin{align} \Pr[\mathcal{O} = \tau | \chi(u_1) = \chi(u_2)] &= \frac{\binom{m-T}{T-\tau} \binom{T}{\tau}}{\binom{m}{T}} \leq (1+o(1)) \frac{\frac{m^{T-\tau}}{(T-\tau)!} \frac{T^{\tau}}{\tau!}}{\frac{m^T}{T!}} \nonumber \\
&=  (T/m)^{\tau} \binom{T}{\tau} \leq \left(\frac{T^2}{m}\right)^\tau \label{eq:overlapestimate_same}
\end{align}

Since $\Pr[\mathcal{O} = \tau | \chi(u_1) = \chi(u_2)] \geq \Pr[\mathcal{O} = \tau | \chi(u_1) \neq \chi(u_2)]$, we have $\Pr[\mathcal{O} = \tau] \leq \left(\frac{T^2}{m}\right)^\tau$. Union bounding, we get 
$$ \Pr[\exists u_1,u_2, \mbox{ that intersect in } > 3c/\epsilon \mbox{ ratings}] \leq \frac{1}{m^c} $$
which proves the first claim.  

Let us turn to (2) now. Let us denote by $\zeta$ a random variable denoting the number movies user 1 has rated in R. By Bayes' rule, we have 
\begin{align*} \Pr[\mathcal{O} = \tau | \chi(u_1) \neq \chi(u_2)] &= \sum_{\tau' \geq \tau} \Pr[\zeta = \tau'] \Pr[\mathcal{O} = \tau | \zeta = \tau', \chi(u_1) \neq \chi(u_2)] \\
\end{align*}

By Chernoff, 
$$\Pr\left[\zeta \in [pT - \sqrt{pT} \log m, pT + \sqrt{pT} \log m]\right] \geq 1 - \exp(-\log^2 m)$$
Furthermore, for $\tau' \in [pT - \sqrt{pT} \log m, pT + \sqrt{pT} \log m]$, and $\tau = O(1)$, we have
\begin{align} \frac{\binom{\tau'}{\tau}\binom{m-\tau'}{T-\tau}}{\binom{m}{T}} &\geq (1-o(1)) \frac{\frac{(\tau')^{\tau}}{\tau!} \frac{(m-\tau')^{T-\tau}}{(T-\tau)!}}{\frac{m^T}{T!}} \nonumber \\
&\geq \frac{(m-\tau')^{T-\tau} (\tau')^{\tau}}{m^T}{\binom{T}{\tau}} \nonumber \\
&\gtrsim (p/2)^{\tau} \left(\frac{T^2}{m}\right)^\tau \label{eq:overlapestimate_diff} \end{align}
where the last inequality follows from $T^2 = o(m)$ and $\tau' \geq pT/2$.  

Putting together \eqref{eq:overlapestimate_same} and \eqref{eq:overlapestimate_diff}, for $\tau = O(1)$, 
$$ \frac{\Pr[\mathcal{O} = \tau | \chi(u_1) \neq \chi(u_2)]}{\Pr[\mathcal{O} = \tau | \chi(u_1) = \chi(u_2)]} = \Theta(1).$$ 
However, $\Pr[\chi(u_1) = \chi(u_2)] = \frac{1}{k}$, so by Bayes' rule
$$ \frac{\Pr[\chi(u_1) \neq \chi(u_2) | \mathcal{O} \geq \tau ]}{\Pr[\chi(u_1) = \chi(u_2) | \mathcal{O} \geq \tau]} \gtrsim k$$ 
which proves (ii).  


\end{proof}

\subsubsection{$T \gg \sqrt{m}$: metric structure reveals latent structure}

We now consider the case of $T = \Omega(\sqrt{m} \log m)$. We show that in this case, nearest neighbors do work -- even if the genre structure is substantially more complicated than in the previous section. We note that this section is not necessary for exhibiting a \emph{separation} between representation learning and nearest neighbors -- however, as in practice, nearest neighbors work reasonably well -- we feel it's important to point out that the reason for that is likely the availability of a larger number of labeled samples. 
In terms of genre structure, as before, every genre will have $m$ movies and we assume the intersection of any two genres will be of size at most $\delta m$ for $\delta = o(1)$. The number of genres a user likes will be $s$, for $s = O(1)$. We show that under these conditions, nearest neighbors do work. 

Intuitively, in this case we wish to show the number of ratings two users share if they don't agree an all their genres is at most $\left(\frac{1}{s}-\frac{1}{s^2}\right) T^2/m$, whereas it is at least $ \frac{1}{s} T^2/m$ when they do. It's easy to check that in expectation, up to $1 \pm o(1)$ factors, these quantities are correct; if additionally they concentrate well, our claim will follow. The concentration bounds are complicated by the fact that the relevant random variables defining these quantities are \emph{not} independent -- to that end, we have to use results from (\cite{panconesi1997randomized}) -- which give Chernoff-type concentration results, for variables that are \emph{negatively associated} in a particular manner.   

We state the precise result here, and relegate the proof to Appendix~\ref{a:largeT}. 

\begin{thm} In the setup of this section, with probability $1 - \exp(-\log^2 M)$ \\
(i) Two users that differ on at least one of the genres they like agree on at most $(1+o(1))(\frac{1}{s} - \frac{1}{s^2}) \frac{T^2}{m}$ ratings. \\  
(ii) Two users that like the same genres share at least $(1-o(1)) \frac{1}{s} T^2/m$ ratings. \\
\label{t:larget}
\end{thm}

The above theorem clearly implies that nearest neighbors will work: we can declare the users that share at least $(1-o(1)) \frac{1}{s}T^2/m$ ratings as neighbors, and predict a missing rating as being 1 if at least one of the neighbors has rated it, and 0 otherwise.

\subsection{Separation from manifold learning} 
\label{subsec:manifold}
We show that common approaches to manifold learning are less powerful than representation learning ---even in the simple linear setting. 
Since various algorithms were proposed for learning a manifold, our argument has to be quite generic -- in fact it applies to every method that starts by constructing a locality graph whose vertices are data-points and edges are put between a datapoint and its nearest neighbors in the ambient space (\cite{belkin2004semi}). We prove something quite strong -- that the nearest-neighbor graph already has already very little information about the labels -- hence no algorithm (spectral or not; even allowed arbitrary computational time) can extract meaningful information \emph{from the graph only}. In this sense, the arguments here can be viewed as strenghtenings of the theorems in the previous section.

{\bf Setup:} We will use the setting from the lower bound on nearest neighbors: each user is fond of a single genre, with the genres overlapping in a $p$ fraction of common movies, which we denote by $R$. Furthermore, each user emits $T$ movies from it's genre, chosen uniformly at random. The neighbors of a given user are all users which share at least $\tau$ of the emitted movies. Recall, $p, k, T$ and $\tau$ are such that $p = O(1)$, $k = \omega(1)$, $T = \Omega(\frac{1}{1-p} \log m)$ and $\tau = O(1)$. Additionally, we will assume that the number of users is such that $N = O(\left(\frac{m}{T^2}\right)^{\tau})$ (this ensures the expected number of nearest neighbors is upper bounded by a constant -- which is realistic in practice.)  
We show the following claims: \\
(1) The distribution $\mathcal{G}_N$ of the graph of nearest neighbors is close to a stochastic block model with $k$ communities \cite{banks2016information}, s.t. the probability of connecting two nodes in the same community is a constant factor larger than the probability of connecting two nodes between communities. The communities correspond to the genres, and a user belongs to the community of the genre he/she likes. \\
(2) Assuming the number of nearest neighbors is $O(1)$, there is no algorithm which can even distinguish the neighborhood graph from an Erd\H{o}s-R\'enyi graph with matching expected degree. (This task is called \emph{weak recovery} or \emph{detection} in the stochastic block model community.) This will follow by applying an information-theoretic lower bounds for the multi-community stochastic block model from \cite{banks2016information}. 

To formalize (1) and (2), we define the notion of a \emph{contiguous} sequence of distributions: 
\begin{defn} A sequence of probability distributions $\mathbb{P}_N$ is \emph{contiguous} to a sequence of distributions $\mathbb{Q}_N$, if for any sequence of events $A_N$,  $\mathbb{Q}_N(A_N) \to 0$ implies $\mathbb{P}_N(A_N) \to 0$. We will denote this as $\mathbb{P}_N \triangleleft \mathbb{Q}_N$.
\end{defn} 

Let $\mbox{SBM}(N,k,q_{\mbox{in}},q_{\mbox{out}})$ be a random graph on $N$ vertices constructed as follows: each vertex is assigned to one of $k$ communities uniformly at random.  Subsequently, for every pair of vertices inside a community, an edge is placed with probability $q_{\mbox{in}}$; for every pair between communities, an edge is placed with probability $q_{\mbox{out}}$; we will be interested in the sparse regime, i.e. $q_{\mbox{in}},q_{\mbox{out}} = O(1/N)$. 

Additionally, let  $\mbox{g-SBM}(N,k,q_{\mbox{in}},q^m_{\mbox{in}}, q_{\mbox{out}}, q^M_{\mbox{out}})$ be a similarly defined distribution of graphs, where edges between pairs of vertices in different communities are be placed with probability $q_{\mbox{out}} \geq q^{m}_{\mbox{out}}$ and between pairs of vertices within communities with probability $q_{\mbox{in}} \leq q^{M}_{\mbox{in}}$, s.t. $q^{m}_{\mbox{out}} \leq q^{M}_{\mbox{in}}$.

Let $\mathbb{ER}(N, \phi)$ an Erd\H{o}s-R\'enyi graph on $N$ vertices with edge probability $\phi$. 

Finally, we will denote $a = \tau \left(T^2/m\right)^{\tau}$, $b =  \frac{1}{2} \left(\frac{p}{2}\right)^{\tau} \left(T^2/m\right)^{\tau}$.  

Proceeding to the formal claims, we first we wish to show that the distribution $\mathcal{G}_N$ of the nearest-neighbor graph is close to a stochastic block model with $k$ communities. The statement is as follows: 

\begin{lem} Consider the distribution $\mathcal{G}_N$ of the graph defined as the nearest-neighbor graph on $N$ users sampled according to the emission model of this section, connecting users whenever they share more than $\tau$ ratings. Then, there is a ``good'' event $\omega_N$ happening with probability $1-o(1)$, s.t. $\mathcal{G}_N$ and $\mbox{g-SBM}(N, G, a, b)$ agree within $\omega_N$.   
\label{l:manifold0}
\end{lem} 
\begin{proof} 


Let's denote by $\mathbb{P}_N$ the distribution on the emitted movies by the $N$ users. First, note that in $\mathbb{P}_N$, conditioned on the number of movies in $R$ for each of the users, the probability of an edge in $\mathcal{G}_N$ between any two users is independent. Moreover, from equation \eqref{eq:overlapestimate_same}, for users in the same genre this probability is $\leq a$. On the other hand, for two users liking different genres, conditioned on the number of movies in $R$ being $pT \pm \sqrt{pT} \log m$, by \eqref{eq:overlapestimate_diff} the probability of an edge between them is $\geq b$. Denoting by $\omega_N$ the event in $\mathbb{P}_N$ that all users have $pT \pm \sqrt{pT} \log m$ movies in $R$, the claim of the Lemma follows.

\end{proof} 

Next, we wish to relate the two models g-SBM and SBM, and show that for the purposes of distinguishing from $\mathbb{ER}$, g-SBM is only more difficult to distinguish from  $\mathbb{ER}$ than SBM is. This intuitively should be clear, as g-SBM allows the probabilities of edges between vertices in different communities to be increased, and the probabilities of edges within the same community to be decreased. The formal statement is as follows:  

\begin{lem} If $\mbox{SBM}(N,k,q_{\mbox{in}},q_{\mbox{out}}) \triangleleft \mathbb{ER}(N, \frac{q_{\mbox{in}} + (k-1)q_{\mbox{out}}}{k}) $ then \\ $\mbox{g-SBM}(N,k,q_{\mbox{in}},q^m_{\mbox{in}},q_{\mbox{out}}, q^M_{\mbox{out}}) \triangleleft \mathbb{ER}(N, \frac{q_{\mbox{in}} + (k-1)q_{\mbox{out}}}{k}) $, for any $q_{\mbox{in}} \leq q^m_{\mbox{in}} \leq  q^M_{\mbox{out}} \leq q_{\mbox{out}}$.  
\label{l:manifold00}
\end{lem} 
\begin{proof}
Let us denote for brevity $\mbox{SBM}(N,k,q_{\mbox{in}},q_{\mbox{out}})$ by SBM, and $\mbox{g-SBM}(N,k,q_{\mbox{in}},a_{\min},q_{\mbox{out}}, b_{\min})$ by g-SBM. For any configuration of edges $E_n \in \{0,1\}^{\binom{N}{2}}$, we claim that $|\mbox{g-SBM}(E_n) - \mathcal{G}_N(E_n)| \leq  |\mbox{SBM}(E_n) - \mathcal{G}_N(E_n)|$. This would clearly imply the claim of the lemma. 

We first marginalize out the assignments $\sigma$ of the vertices into communities: 
\begin{align*} |\mbox{g-SBM}(E_n) - \mathbb{ER}(E_n)| = \left|\sum_{\sigma \in [k]^N} \frac{1}{k^N} \left(\mbox{g-SBM}(E_n | \sigma) - \mathbb{ER}(E_n)\right)\right|  \end{align*}   
Next, we note that for any assignment $\sigma \in [k]^N$, $|\mbox{g-SBM}(E_n | \sigma) - \mathbb{ER}(E_n)| \leq |\mbox{SBM}(E_n | \sigma) - \mathbb{ER}(E_n)|$: indeed, for pairs $(i,j)$ of vertices in the same community, $\Pr[(i,j) \in \mbox{SBM} | \sigma] > \Pr[(i,j) \in \mathbb{ER}]$, and for pairs $(i,j)$ in different communities $\Pr[(i,j) \in \mbox{SBM} | \sigma] < \Pr[(i,j) \in \mathbb{ER}]$; since $\Pr[(i,j) \in \mbox{g-SBM} | \sigma] < \Pr[(i,j) \in \mbox{SBM} | \sigma]$ in the former case, and $\Pr[(i,j) \in \mbox{g-SBM} | \sigma] > \Pr[(i,j) \in \mbox{SBM} | \sigma]$ in the latter, the claim follows.  

\end{proof} 

To formalize point (2), we will apply results from \cite{banks2016information} to show that for the parameter range we are interested in according to the above Lemmas, SBM and $\mathbb{ER}$ are indistinguishable. 

\begin{lem} [Follows from \cite{banks2016information}] $ \mathbb{ER}(N, \frac{q_{\mbox{in}} + (k-1)q_{\mbox{out}}}{k}) \triangleleft \mbox{SBM}(N,k,a, b)$. 
\label{l:manifold1}
\end{lem}
\begin{proof} 
The result will follow from Theorem 1 in \cite{banks2016information}. Namely, according to (11) in that theorem,  $\mbox{SBM}(N,G,a, b)$ is contiguous to an  Erd\H{o}s-R\'enyi graph on $N$ vertices with edge probability $\phi$ if $$N (a - b)^2 \lesssim (a + (k-1) b) \log k$$ Plugging in $a = \tau \left(\frac{T^2}{m}\right)^{\tau}, b = \frac{1}{2} \left(\frac{p}{2}\right)^{\tau} \left(\frac{T^2}{m}\right)^{\tau}$, this condition translates to 
\begin{equation} \left(\frac{T^2}{m}\right)^{\tau} \left(\tau - \frac{1}{2} \left(\frac{p}{2}\right)^{\tau}\right)^2 < \frac{1}{N} \left(\tau + (k-1)\frac{1}{2} \left(\frac{p}{2}\right)^{\tau}\right) \log k \label{eq:banks} \end{equation}   
But by our assumption $\left(\frac{T^2}{m}\right)^{\tau}  = \Theta(1/N), p = O(1)$, and $k = \omega(1)$, so \eqref{eq:banks} is satisfied. Hence, the theorem statement follows. 
\end{proof}

Putting Lemma~\ref{l:manifold0}, \ref{l:manifold00} and \ref{l:manifold1} together, we get as a consequence the main result that no algorithm can distinguish between the nearest neighbor graph and an Erd\H{o}s-R\'enyi graph of matching expected degree. 


\begin{thm}  The distribution $\mathcal{G}_N$ of the graph defined as the nearest-neighbor graph on $N$ users sampled according to the emission model of this section, connecting users whenever they share more than $\tau$ ratings, is contiguous to an Erd\H{o}s-R\'enyi graph on $N$ vertices with edge probability $\phi = \frac{a + (k-1)b}{k}$. Consequently, no statistical test that takes as input a graph sampled from one of these two distributions (say with probability 1/2) can output which distribution the sample came from error probability $o(1)$. 
\label{t:manifold}  
\end{thm} 
\begin{proof}
Contiguity follows immediately from Lemmas ~\ref{l:manifold0}, \ref{l:manifold00} and \ref{l:manifold1}. The consequence to distinguishability is a standard result (e.g. Claim 2.2 in \cite{perry2016optimality}).     
\end{proof} 


\subsection{Semi-supervised learning using encoders}
\label{subsec:semisupervised}
We give a further concrete example of the practical benefit of representation learning, which is that it can provably lower the sample complexity of classification tasks. Namely, we show that given an encoder (that can be learnt using unlabeled examples), the supervised task can be run with very few labeled samples. Furthermore, we can show that any method that uses only labeled samples would need substantially more. 

\noindent {\bf Setup:} Consider again the simple mixture model, with the same notation as before: the number of genres is $k$, the number of ratings per user $T = \Omega(\log M)$, and the number of genres per user $s$. Furthermore, suppose we are interested in classification, where the label is produced by a hyperplane: 
$l(h) = \mbox{sgn}(\langle w, 2h - 1 \rangle)$.   

\subsubsection{Semi-supervised setting: sample complexity with access to an encoder}
We consider first the semi-supervised setting: namely if an algorithm has access to an encoder (hence to the latent genres), after seeing a very small number users it can predict $l(h)$ very well. (Of course, the encoder can be trained in an unsupervised manner, using the algorithms in Section~\ref{subs:linear}).

First as a warmup, we show if the number of possible $h$'s is small, with an access to an encoder, after seeing a small number of users, an algorithm can in fact predict perfectly. More formally: 
\begin{prop} There exists a polynomial-time algorithm $\mathcal{A}$ that given access to a $(1,1)$-valid encoder, with high probability after $t = O(k^s \log M)$ samples $(x_1, l(h_1)), (x_2, l(h_2)), \dots, (x_t, l(h_t))$ can achieve
$$\Pr_{h_{t+1}}\left(\mathcal{A}\left((x_1,l(h_1)), (x_2,l(h_2)), \dots, (x_t, l(h_t)), x_{t+1}\right) = l(h_{t+1}) \right) = 1$$
\end{prop}
\begin{proof} 
After $O(k^s \log M)$ samples, $\mathcal{A}$ with high probability has access to the value of $\mbox{sgn}(\langle w, 2h-1 \rangle)$ for all possible $h$ 
so $\mathcal{A}$ can just build a lookup table. 
\end{proof}

In fact, it's quite easy to generalize the above proposition using standard margin generalization theory to an arbitrary $(\beta, \gamma)$-encoder (rather than a (1,1)-encoder), and a similar setting. 

\begin{thm} Suppose that $\mathcal{O}$ is a $(\beta, \gamma)$-valid encoder with respect to the norm $\|\cdot \|_2$, and in addition it always outputs a vector of $l_2$ norm at most $B$. Furthemore, suppose that the linear predictor $\mbox{sgn}(\langle w,2h-1 \rangle)$ enjoys margin $\rho$. Then, there exists a polynomial-time algorithm $\mathcal{A}$ that given access to $\mathcal{O}$, with probability $1-\delta$ after $t$ samples $(x_1, l(h_1)), (x_2, l(h_2)), \dots, (x_t, l(h_t))$ can achieve
\begin{align*} &\Pr_{x_{t+1}} \left[ \mathcal{A}\left((x_1,l(h_1)), (x_2,l(h_2)), \dots, (x_t, l(h_t)), x_{t+1}\right) \neq l(h_{t+1})\right] \leq C_t + R_t + E_t\end{align*}
where $C, R$ and $E$ are defined as
\begin{align*} & C_t = (1-\beta)\left(1 + \sqrt{\frac{\ln 1/\delta}{(1-\beta) t}}\right) \rho B + \beta\left(1 - \sqrt{\frac{\ln 1/\delta}{\beta t}}\right) \rho \gamma, \hspace{4mm} R_t = \sqrt{\frac{\rho^2}{t}}, \hspace{4mm} E_t = \sqrt{\frac{\ln 1/\delta}{t}} 
\end{align*}
\label{t:semisupervised}
\end{thm}

To parse the right-hand side above, note that $\beta \approx 1$ for a good encoder, so $C_t$ is small; $R_t$ comes from standard Rademacher bounds for linear predictors with a margin, and $E_t$ is the usual generalization term. Moreover, if $\rho$ is thought of as being a constant, $t = O(\ln 1/\delta)$ suffices to make the bound non-trivial. 

\begin{proof}

The algorithm for $\mathcal{A}$ is straightforward: a soft-margin SVM, using the encoder $\mathcal{O}$ to calculate surrogates for $h$. More formally, we solve 
$$\min_{\|w\|_2 \leq \rho} \sum_{t'=1}^t \max\left(0, 1 - l(x_{t'})\langle w, \mathcal{O}(x_{t'})\rangle\right)  $$  
Let $S = \{x_1, x_2, \dots, x_t\}$ be the set of samples, and let us denote $L(w) = \Pr_x [\mbox{sgn}(\langle w, x \rangle) \neq l(x)]$, $\hat{L}(w,S) = \Pr_{x \in S}[\mbox{sgn}(\langle w, x \rangle) \neq l(x)]$ and  
$C(w,S) =  \frac{1}{|S|} \sum_{x \in S} \max\left(0, 1 - l(x)\langle w, \mathcal{O}(x)\rangle\right) $.   

By Theorem 7 of \cite{bartlett2002rademacher}, we have that with probability $1-\delta$ over the choice of $S$,  
\begin{equation} \forall w, L(w) \lesssim C(w, S) +  R_t(\mathcal{F}) + \sqrt{\frac{\ln 1/\delta}{t}}  \label{eq:bartlett} \end{equation}
where $\mathcal{F}$ is the set of linear predictors with $l_2$ weights bounded by $\rho$ and $R_t(\mathcal{F})$ denotes the empirical Rademacher complexity of $\mathcal{F}$. 
The theorem statement will follow from inspecting $C(w,S)$ and $R_t$ in this expression. 

We consider $C(w,S)$ first. We show that with probability $1-\delta/2$ over the choice of $S$, 
\begin{equation} \min_{\|w\|_2 \leq \rho} C(w,S) \leq   (1-\beta)\left(1 + \sqrt{\frac{\ln 1/\delta}{(1-\beta) t}}\right) \rho M + \beta\left(1 - \sqrt{\frac{\ln 1/\delta}{\beta t}}\right) \rho \gamma \label{eq:surrogate} \end{equation} 
Indeed, let $\Delta$ be the set of $x \in S$, s.t. $\|\mathcal{O}(x) - h\|_2 > \gamma$. Since the encoder always outputs a vector of norm at most B, for the points in $\Delta$ we have 
\begin{equation} \max\left(0, 1 - l(x)\langle w, \mathcal{O}(x)\rangle\right) \leq \rho B \label{eq:bad}\end{equation} 
On the other hand, for the points in $S/\Delta$, since $\max\left(0, 1 - l(x) \langle w, h \rangle\right)$ is $\|w\|_2$-Lipschitz as a function of $h$, 
\begin{equation} \max\left(0, 1 - l(x)\langle w^*, \mathcal{O}(x)\rangle\right) \leq \max\left(0, 1 - l(x)\langle w^*, h\rangle\right) + \rho \gamma = \rho \gamma \label{eq:good} \end{equation} 
where the last equality holds since we assume the points $(x,l(x))$ in $S/\Delta$ are separable with margin $\rho$. 
Additionally, by Chernoff, with probability $1-\delta/2$, $|\Delta| \lesssim (1-\beta)\left(1 + \sqrt{\frac{\ln 1/\delta}{(1-\beta) t}}\right) t $ and $|S/\Delta| \gtrsim \beta\left(1 - \sqrt{\frac{\ln 1/\delta}{\beta t}}\right) t $. 
Putting this together with \eqref{eq:bad} and \eqref{eq:good} we get \eqref{eq:surrogate}. 

Finally, it's a standard fact that $ R_t(\mathcal{F}) \lesssim \sqrt{\frac{\rho^2}{t}}$, which finishes the proof. 
  


\end{proof} 


\subsubsection{Supervised setting: low advantage over random guessing} 

On the other hand, in the purely supervised setting (i.e. without knowing the genres), we show that no algorithm can do better than essentially random guessing. Intuitively, if the genres are chosen by randomly partitioning the movies, after the algorithm has seen $t$ users for a small $t$, it has no knowledge of the genre membership of all but $tT$ of the movies; by switching the order of the randomness, and generating the genre structure as the algorithm proceeds, there is nothing better to do than random guessing. 

We formalize this intuition as follows: 
 
\begin{thm} Suppose the genres are chosen by partitioning the movies into $k$ disjoint genres of size $M/k$ uniformly at random. Furthermore, assume that the linear predictor corresponding to $w$ is balanced, i.e. $E_{h} [\mbox{sgn}(\langle w, 2h-1\rangle)] = 0$. Then, with high probability over the choice of the partitioning and the $t$ samples $(x_i, l(h_i)), i \in [1,t]$, any algorithm $\mathcal{A}$ has 
$$\Pr_{x_{t+1}}\left(\mathcal{A}((x_1, l(h_1)), (x_2, l(h_2)), \dots, (x_t, l(h_t)), x_{t+1}) = l(h_{t+1}) \right) \leq \frac{1}{2} + O(s t T^2/k)$$
\label{t:randomguess}
\end{thm}

The proof proceeds by showing that even if $\mathcal{A}$ has access to $w$ and $h_1, h_2, \dots, h_t$ in addition to $(x_i, l(h_i)), i \in [1,t]$ the statement still holds true -- intuitively since the user simply has observed the genre membership of too few movies. 
More precisely, we show: \\
(i) The probability of a user emitting at least one of the already emitted movies is $O(s tT^2/k)$. \\
(ii) Conditioned on a user not emitting at least one of the already emitted movies, all $h_{t+1}$ are equally likely. \\ 
The full proof is relegated to Appendix~\ref{a:randomguess}.

\section{Conclusion}

The paper has tried to formalize representation learning and used the framework to show that even in in simple settings such as linear models, simple algorithms for representation learning are more powerful than standard techniques such as nearest neighbors and manifold learning. 

 The definitions and framework proposed are meant as a first cut - e.g. conceivably there is not a {\em single} high level representation, but multiple levels of representation which capture different levels of understanding/utility. It may even be fallacious to think of these different representations as a simple hierarchy. 
  Likewise, the probabilistic/Bayesian framework was chosen for reasons of convenience and familiarity. One could explore non-Bayesian descriptions, but then articulating a {\em goal} or {\em utility} of the representation seems more challenging. We hope this will stimulate further theoretical study of representation learning, since it powerfully extends existing notions in machine learning, even in fairly simple settings. Generalizing our study to more complicated settings is left for future work.

\bibliographystyle{plainnat}
\bibliography{bibliography}

\newpage
\appendix 
\section{Proof of Theorem~\ref{t:wordembed}} 
\label{s:missingword}

\begin{proof}[Proof of Theorem~\ref{t:wordembed}]


As outlined in the proof sketch in Section~\ref{s:loglinear}, we will show that with probability $1-\exp(-\log^2 M)$ over the choice of movie vectors $W_x$, and the observable $(x_1, x_2, \dots, x_T) \in \{0,M\}^T$, 

$$\bigg\langle \frac{\sum_{i=1}^T W_{x_i}}{\|\sum_{i=1}^T W_{x_i}\|}, h \bigg\rangle \geq 1-o(1) $$

which clearly implies the theorem statement. 

Because of the distributional assumption on the movie vectors and latent vectors $h$, without loss of generality we assume $h= e_1$.  
For notational convenience, we will denote $\E_{W_x}[\cdot]$ the expectation with respect to the movie vectors, and $\E_{p(x|h)}[\cdot]$ the expectation with respect to the 
conditional distribution on $h$ (which is a random quantity due to the randomness of the movie vectors). 

We will show that with high probability under the choice of movie vectors and the distribution $P(x|h)$, when the parameters satisfy the conditions in the theorem statement, we have:
\begin{equation} \sum_{i=1}^T (W_{x_i})_1 \geq (1-o(1)) T \frac{B^2}{4d} \label{eq:denom} \end{equation}
Moreover, we will show that this implies that with high probability,
\begin{equation} \left\|\sum_i W_{x_i} \right\|  \leq (1+o(1)) \left(\sum_{i=1}^T (W_{x_i})_1\right)  \label{eq:num} \end{equation}
This clearly implies the statement of the lemma.

Let's start with \eqref{eq:denom}. Considering the randomness of $p(x|h)$ only, $ \sum_{i=1}^T (W_{x_i})_1 $ is a sum of $T$ iid random variables. So, towards applying Chernoff we first analyze 
$$\E_{p(x|h)}[(W_x)_1] $$  
By definition of the distribution $p(x|h)$, this quantity is 
$$ \sum_{x \in M} \frac{\exp((W_x)_1)}{Z} (W_x)_1$$ 

By Lemma 2.1. in \cite{arora2016latent}, with probability $1 - \exp(-\log^2 M)$, we have $Z = E_{\mathbf{m}}[Z](1 \pm o(1))$ 
where the expectation is taken with respect to the randomness of the movie vectors. 
Additionally, the expectation of $Z$ can be calculated explicitly: 
$$ \E_{W_x}[Z] =  M \int_{-\infty}^{\infty} \sqrt{\frac{d}{2 \pi}}\exp(B y) \exp(-2 y^2 d) dy =  M \exp(B^2 d/8) $$  
Now, let us denote by $\mathcal{F}$ the event $Z = E_{W_x}[Z](1 \pm o(1)$. Then, by definition of conditional expectation we have 
$$ \E_{W_x}\left[\frac{\exp((W_x)_1)}{Z} (W_x)_1 \Bigg| \mathcal{F}\right] = \frac{1}{\Pr[\mathcal{F}]} \E_{W_x}\left[\frac{\exp((W_x)_1)}{Z} (W_x)_1 \mathbf{1}_{\mathcal{F}}\right] $$ 



We focus on $\E_{W_x}[\frac{\exp((W_x)_1)}{Z} (W_x)_1 \mathbf{1}_{\mathcal{F}}] $ now. Denoting by $\mathcal{P}$ the event $(W_x)_1 > 0$, we can write 

\begin{align*} & \E_{W_x}\left[\frac{\exp((W_x)_1)}{Z} (W_x)_1 \mathbf{1}_{\mathcal{F}} \mathbf{1}_{\mathcal{P}}\right] + \E_{W_x}\left[\frac{\exp((W_x)_1)}{Z} (W_x)_1 \mathbf{1}_{\mathcal{F}} \mathbf{1}_{\bar{\mathcal{P}}}\right] \geq \\ 
& \E_{W_x}\left[\frac{\exp((W_x)_1)}{Z} (W_x)_1 \mathbf{1}_{\mathcal{F}} \mathbf{1}_{\mathcal{P}}\right] + \E_{W_x}\left[\frac{\exp((W_x)_1)}{Z} (W_x)_1 \mathbf{1}_{\bar{\mathcal{P}}}\right]  \end{align*}

Focusing on the first term, it is easy to see  
$$ \E_{W_x}\left[\frac{\exp((W_x)_1)}{Z} (W_x)_1 \mathbf{1}_{\mathcal{F}} \mathbf{1}_{\mathcal{P}}\right] \geq \int_{0}^{t} \sqrt{\frac{d}{2 \pi}} \exp(B y) B y \exp(-2 y^2 d) dy$$ 
where $t$ is chosen such that $\int_{t}^{\infty}  \sqrt{\frac{d}{2 \pi}} \exp(-2 y^2 d) dy = 1-\Pr[\mathcal{F}]$. By simple Gaussian tail bounds,  $t = O(\log n/d)$. But then, simple algebraic rewriting shows 
$$ \int_{0}^{t} \sqrt{\frac{d}{2 \pi}}\exp(B y) B y \exp(-2 y^2 d) dy = \exp(B^2 d/8) B \int_{0}^{t} \sqrt{\frac{d}{2 \pi}} y \exp(-2d(y - B/4d)^2) dy$$ 
Additionally we have: 
\begin{align*} & \int_{t}^{\infty} \sqrt{\frac{d}{2 \pi}} y \exp(-2d(y - B/4d)^2) dy \\
& = \int_{0}^{\infty} \sqrt{\frac{d}{2 \pi}} y \exp(-2d(y - B/4d)^2) dy - \int_{t}^{\infty} \sqrt{\frac{d}{2 \pi}} x \exp(-2d(y - B/4d)^2) dy \\ 
& \geq \int_{0}^{\infty} \sqrt{\frac{d}{2 \pi}} y \exp(-2d(y - B/4d)^2) dy  - \exp(-\log^2 n) \end{align*} 
Hence, we get 
\begin{align*} & \E_{W_x}\left[\frac{\exp((W_x)_1)}{Z} (W_x)_1 \mathbf{1}_{\mathcal{F}} \mathbf{1}_{\mathcal{P}}\right] + \E_{W_x}\left[\frac{\exp((W_x)_1)}{Z} (W_x)_1 \mathbf{1}_{\bar{\mathcal{P}}}\right]  \\
& \geq (1-o(1)) \frac{B}{\E[Z]}\left(\int_{-\infty}^{\infty} \sqrt{\frac{d}{2 \pi}} y \exp(-2d(y - B/4d)^2) dy - \exp(-\log^2n)\right) \\
& \geq (1-o(1)) \frac{B}{\E[Z]} \exp(B^2 d/8) \left( B/4d\right) \\
& = (1-o(1)) \frac{1}{M} B^2/4d \end{align*}
where the last inequality follows since $\int_{-\infty}^{\infty} \sqrt{\frac{d}{2 \pi}} y \exp(-2d(y - B/4d)^2) dy$ is the mean of a Gassian with mean $\frac{B}{4d}$ and variance $\frac{1}{d}$. 

Hence, with high probability over the word vectors, 
$$ \sum_{x \in M} \frac{\exp((W_x)_1)}{Z} (W_x)_1 \geq B^2/4d $$  
A Chernoff bound then implies \eqref{eq:denom} is true with high probability. 

We consider \eqref{eq:num} now. We have
$$\left\|\sum_i W_{x_i} \right\|^2 =  \sum_{j=1}^d \left(\sum_{i=1}^T (W_{x_i})_j\right)^2  $$ 
Let us split in two cases, $j=1$, and $j \neq 1$. 
Consider $j=1$ first. By \eqref{eq:denom}, it directly follows that $\left(\sum_{i=1}^T (W_{x_i})_1\right)^2 \geq (1-o(1)) \left(T \frac{B^2}{4d}\right)^2$ with high probability. 

%

So we can focus on $j \neq 1$. 
Same as before, towards applying Chernoff we first analyze 
$$\E_{p(x|h)}[(W_x)_j] $$  
By definition of the distribution $p(x|h)$, this quantity is 
$$ \sum_{x \in M} \frac{\exp((W_x)_1)}{Z} (W_x)_j$$ 
However, note that $j \neq 1$, under the randomness of $W_x$, $(W_x)_j$ and $\frac{\exp((W_x)_1)}{Z}$ are independent variables. Moreover, with high probability it holds that 
$\frac{\exp((W_x)_1)}{Z} \leq (1 + 2 \frac{B \log M}{\sqrt{d}})$.  Hence, picking the randomness of the first coordinate first and $j$ next, 
$$ \sum_{x \in M} \frac{\exp((W_x)_1)}{Z} (W_x)_j $$ 
is a sum of $M$ independent random variables with mean 0 and variance at most $O(1/d)$ since $1 + 2 \frac{B \log M}{\sqrt{d}} = O(1)$. 
Hence, by Chernoff, with high probability 

$$ \left |\sum_{x \in M} \frac{\exp((W_x)_1)}{Z} (W_x)_j \right| \leq \sqrt{1/(Md)} $$ 
This implies that  $ \displaystyle \left(\sum_{i=1}^T (W_{x_i})_j\right)^2 \leq T^2/(Md) $. Hence, 
%
%
$\|\sum_i W_{x_i} \| \leq (1+o(1)) (T B^2/4d)$ if $M = \Omega(d^2/B^4)$ as we need. This finishes the proof of \eqref{eq:num} and the lemma overall.

\end{proof}
\section{Proof of Theorem~\ref{t:larget}}
\label{a:largeT}

%
%
%
%
\begin{proof}

Let us proceed to (i) first. Consider two users $u_1, u_2$ that differ on at least one of the genres they like. Let $I_{i}, i \in [M]$ be a $0-1$ random variable, s.t. $I_{i}=1$ both $u_1, u_2$ have rated movie $i$. Then, it suffices to show that $\sum_{i} I_{i} \leq  (1+o(1)) (\frac{1}{s}-\frac{1}{s^2}) \frac{T^2}{m} $ with probability $1 - \exp(-\log^2 M)$.

Let $m_1$ be the size of the union of all the genres user 1 likes, and $m_2$  the size of the union of all the genres user 2 likes. Note that $m_1, m_2 \geq (1-o(1)) ms$, since $s = O(1)$ and the overlap of any two genres is $o(m)$. 

Furthermore, for any movie $i$ that belongs to at least one of the genres of both $u_1, u_2$, $ \E[I_i] = \frac{\binom{m_1-1}{T-1}}{\binom{m_1}{T}} \frac{\binom{m_2-1}{T-1}}{\binom{m_2}{T}} = \frac{T}{m_1} \frac{T}{m_2}$. For any other movie $i$, $\E[I_i] = 0$. From this, we get

$$ \E\left[\sum_i I_{i}\right] \leq  (1+o(1)) (s-1) m\left(\frac{T}{sm}\right)^2 = (1+o(1)) \left(\frac{1}{s} - \frac{1}{s^2}\right) \frac{T^2}{m} $$ 

On the other hand, we claim the variables $I_i$ are ``negatively associated'' in the following way:  
\begin{equation} \forall U \subseteq [m], \Pr[\Pi_{i \in U}I_i] \leq \Pi_{i \in U} \Pr[I_i]  \label{eq: pos} \end{equation}
By~\cite{panconesi1997randomized}, the usual Chernoff-type upper tail bounds would hold if this were satisfied. 
%
If at least one movie $i \in U$  doesn't belong to at least one of the genres of both $u_1, u_2$, the claim is trivial, so we may assume this is not the case. Then, 
$$\Pr[\Pi_{i \in U}I_i]   = \frac{\binom{m_1-|U|}{T-|U|}}{\binom{m_1}{T}} \frac{\binom{m_2-|U|}{T-|U|}}{\binom{m_2}{T}}$$ 
Since  $\forall i \in U, \Pr[I_i] =  \frac{\binom{m_1-1}{T-1}}{\binom{m_1}{T}} \frac{\binom{m_2-1}{T-1}}{\binom{m_2}{T}}$, it suffices to show 
$\binom{m_i-1}{T-1} \geq \binom{m_i-|U|}{T-|U|}, i \in \{1,2\}$. However, this follows since 
\begin{align*} \binom{m_i-1}{T-1}  &= \binom{m_i-|U|}{T-|U|} \Pi_{j=1}^{|U|-1}\frac{m_i-|U|+j}{T-|U|+j}  \end{align*}
and $\frac{m_i-|U|+j}{T-|U|+j} \geq 1, \forall j \in [1,|U|-1]$.  
%
%

We proceed to (ii). Let $L$ be the number of movies in the union of the genres of the users and note $L \leq sm$. 
Consider any two users $u_1, u_2$. For any subset $S_1, |S_1| = T$ of movies, we will prove that conditioned on $u_1$ having seen the movies in $S_1$, with probability $1 - \exp(-\log^2 M)$, the users will share at least $(1-o(1)) \frac{1}{s} T^2/m$ ratings. 

Similarly as before, let $I_i, i \in S_1$ be a 0-1 random variable, s.t. $I_i=1$ if $u_2$ has rated movie $i$. Then, 
$\E[I_i] =  \frac{T}{L}$ and $\E[\sum_{i \in S_1} I_i] = \frac{T^2}{L} \geq \frac{T^2}{m}\frac{1}{s} $
Furthermore, we claim the variables $I_i$ are negatively associated in the following manner: 
\begin{equation} \forall U \subseteq [m], \Pr[\Pi_{i \in U}\bar{I_i}] \leq \Pi_{i \in U} \Pr[\bar{I_i}]  \label{eq: neg} \end{equation}
By~\cite{panconesi1997randomized} again, the usual Chernoff-type lower tail bounds would hold if this were satisfied. We have
$\Pr[\Pi_{i \in U}\bar{I_i}]   = \frac{\binom{L-|U|}{T}}{\binom{L}{T}}$. Since $\forall i \in U, \Pr[\bar{I_i}] =  \frac{\binom{L-1}{T}}{\binom{L}{T}}$, it suffices to show 
$\binom{L-1}{T-1} \geq \binom{L-|U|}{T-|U|}, i \in \{1,2\}$. However, this follows since 
\begin{align*} \binom{L-1}{T}  &= \binom{L-|U|}{T} \Pi_{j=1}^{|U|-1}\frac{L-|U|+j}{L-T-|U|+j} \end{align*}
and $\frac{L-|U|+j}{L-T-|U|+j} \geq 1$. 
%
%
%
%
 %
\end{proof}  
\section{The metric structure of nearest neighbors with independent emissions}
\label{s:independent} 


 

\begin{thm} In the setup of this section, \\
(i)  For any constant $c$, if the number of users is at most $m^c$, with probability at least $1 - \frac{1}{m^c}$, the number of ratings two users share is at most $3c/\epsilon$. \\
(ii) For any $\tau \leq \frac{3c}{\epsilon}$, 
    $$\Pr[\chi(u_1) \neq \chi(u_2) | u_1, u_2 \mbox{ share at least } \tau \mbox{ ratings}] \gtrsim k \Pr[\chi(u_1) = \chi(u_2) | u_1, u_2 \mbox{ share at least } \tau \mbox{ ratings}]$$  
 
\end{thm} 
\begin{proof}

Let us prove (i) first. 

Let $u_1$, $u_2$ be two users and let $\chi(u_1), \chi(u_2)$ be their respective genres. Let $\mathcal{O}$ be a random variable denoting the size of the overlap of the ratings of $u_1$ and $u_2$. Then, we have
\begin{equation} \Pr[\mathcal{O} = \tau | \chi(u_1) = \chi(u_2)] = \binom{T}{\tau}^2 \left(\frac{1}{m}\right)^{\tau} \left(1-\frac{T-\tau}{m}\right)^{2(T-\tau)} \label{eq:ioverlapsame}\end{equation}

Indeed, the $\binom{T}{\tau}^2$ counts all possible $\tau$ locations in which the users overlap; $(\frac{1}{m})^{\tau} (1-\frac{T-\tau}{m})^{2(T-\tau)}$ is the probability that the users agree in those $\tau$ locations, and only there. 

Since $\binom{T}{\tau}^2 \leq T^{2\tau}$, we get
$$ \Pr[\mathcal{O} = \tau | \chi(u_1) = \chi(u_2)]  \leq \left(\frac{T^2}{m}\right)^{\tau} $$

But $\Pr[\mathcal{O} = \tau | \chi(u_1) = \chi(u_2)] \geq \Pr[\mathcal{O} = \tau | \chi(u_1) \neq \chi(u_2)]$, so $\Pr[\mathcal{O} = \tau] \leq \left(\frac{T^2}{m}\right)^\tau$. Union bounding, we get 
$$ \Pr[\exists u_1,u_2, \mbox{ that intersect in } > 3c/\epsilon \mbox{ ratings}] \leq \frac{1}{m^c} $$
which proves the first claim.  

Let us turn to (ii) now. Reasoning analogously as for \eqref{eq:ioverlapsame}, we get
\begin{align*} \Pr[\mathcal{O} = \tau | \chi(u_1) \neq \chi(u_2)] \geq  \binom{T}{\tau}^2 \left(\frac{p}{m}\right)^{\tau} \left(1-\frac{T}{m}\right)^{2(T-\tau)}
\end{align*}

However, since $T^2 = o(m)$, we have $(1-\frac{T}{m})^{2(T-\tau)} \geq 1 - o(1)$, so 
$$ \Pr[\mathcal{O} = \tau | \chi(u_1) \neq \chi(u_2)]  \geq  p^{\tau} \Pr[\mathcal{O} = \tau | \chi(u_1) = \chi(u_2)] $$ 
But, since $\tau = O(1)$, we get  
$$ \frac{\Pr[\mathcal{O} = \tau | g(u_1) \neq g(u_2)]}{\Pr[\mathcal{O} = \tau | g(u_1) = g(u_2)]} = \Theta(1).$$ 
Additionally, $\Pr[g(u_1) = g(u_2)] = \frac{1}{k}$, so by Bayes' rule
$$ \frac{\Pr[g(u_1) \neq g(u_2) | \mathcal{O} \geq \tau ]}{\Pr[g(u_1) = g(u_2) | \mathcal{O} \geq \tau]} \gtrsim k$$ 
which proves (ii).  


\end{proof}

More precisely, we will prove the following theorem: 
\begin{thm} In the setup of this section, with probability $1 - \exp(-\log^2 M)$ \\
(i) Two users that differ on at least one of the genres they like agree on at most $(1+o(1)) \left(\frac{1}{s}-\frac{1}{s^2}\right) T^2/m$ ratings. \\  
(ii) Two users that like the same genres share at least $(1-o(1))\frac{1}{s} T^2/m $ ratings. \\
\label{t:indep_larget}
\end{thm} 

\begin{proof}

The proof proceeds similarly as the proof of Theorem~\ref{t:larget}. 

Let us proceed to (i) first. Let $I_{i,t,t'}, i \in [k]$ be a $0-1$ random variable, which is $1$ if $u_1$ emitted movie $i$ at position $t$ and $u_2$ emitted movie $i$ at position $t'$. Then, the size of the overlap between the users is certainly upper bounded by $\sum_{t,t',i} I_{i,t,t'} $. 
Let $m_1$ be the size of the union of all the genres user 1 likes, and $m_2$  the size of the union of all the genres user 2 likes. Then, for any movie $i$ that belongs to at least one of the genres of both $u_1, u_2$, 
$$ \E[I_{i,t,t'}] = \frac{1}{m_1} \frac{1}{m_2}$$
For the rest of the movies, $\E[I_{i,t,t'}] = 0$ otherwise. Since each of the genres has $m$ movies, we get 
$$ \E\left[\sum_{i,t,t'} I_{i,t,t'} \right] \leq  (1+o(1)) (s-1) m\left(\frac{T}{sm}\right)^2 = (1+o(1)) \left(\frac{1}{s}-\frac{1}{s^2}\right) \frac{T^2}{m} $$ 

However, it's obvious that the variables $I_{i,t,t'}$ satisfy $\forall U \subseteq [k] \times [T] \times [t]$, 
$$\Pr[\Pi_{u \in U} I_{u}] \leq \Pi_{i \in U} \Pr[I_u] $$ 
so by ~\cite{panconesi1997randomized}, the usual Chernoff upper bounds hold, which implies the statement. 

We proceed to (ii) next. Let $u_1, u_2$ be two users that share all genres.  Let us denote by $B$ the set of movies $i$ which belongs to a single genre. 
We will show that with probability $1-\exp(-\log^2 M)$, the number of movies in $B$ that were rated exactly once is at least $(1-o(1))T$. From this, the claim is immediate: namely, conditioned on the number of movies in $B$ that were rated being $l$, the movies themselves are uniform are uniformly distributed among all $l$-sized subsets of $B$. Since the size of $B$ is $(1-o(1))sm$, we can apply (ii) of Theorem~\ref{t:larget} to get the conclusion of the theorem. 

Returning to the desired claim, let $I_i$ be a $0-1$ random variable, which is 1 if the movie $i$ was emitted more than once. We claim that with probability $1-\exp(-\log^2M)$, $\sum_{i \in B} I_i$ is at most $T^2/m$. Namely, note that $\Pr[I_i | I_U] \leq \frac{T}{m}$,  for any subset $U \subset B$. By Lemma 1.19 in \cite{auger2011theory}, it holds that with probability $1-\exp(-\log^2 M)$, $\sum_{i \in B} I_i \leq (1+o(1)) \frac{T^2}{m} = o(T)$. Also, by Chernoff, with probability $1-\exp(-\log^2M)$, the number of movies in $B$ that appear at least once is $(1-o(1))T$. Union bounding, the number of movies in $B$ that appear exactly once is at least $(1-o(1))T$. 


 
\end{proof}  
\section{Proof of Theorem~\ref{t:randomguess}} 
\label{a:randomguess}

\begin{proof}
The theorem holds vacuously if $stT^2/k > 1/2$, so we may assume otherwise. 

We will prove an even stronger version of the claim of the theorem: namely if $\mathcal{A}$ even has access to $w$ and $h_1, h_2, \dots, h_t$ in addition to $(x_i, l(h_i)), i \in [1,t]$ the statement still holds true. 
The theorem will follow from the following two claims: \\
(i) The probability of a user emitting at least one of the already emitted movies is $O(s tT^2/k)$. \\
(ii) Conditioned on a user not emitting at least one of the already emitted movies, all $h_{t+1}$ are equally likely. 

Claims (i) and (ii) imply the theorem statement since $\mbox{sgn}(\langle w, 2h - 1 \rangle) = - \mbox{sgn}(\langle w, 2(1-h) - 1 \rangle)$, and hence 
$$\Pr_{h}[\mbox{sgn}(\langle w, 2h - 1 \rangle) = 1] = \Pr_{h}[\mbox{sgn}(\langle w, 2h - 1 \rangle) = -1] = \frac{1}{2}$$ 
where the probability is over the choice of a random $h$.
	
We proceed to (i) first. Since the number of movies rated by the first $t$ users is at most $tT$, the probability of a user rating at least one of the $tT$ seen movies is $O(s tT^2/k)$. 



Next, we prove (ii). 
Towards that, let's denote by $g(m)$ the random variable for the genre of movie $m$, and let us overload the notation to denote by $g(U) := (g(m): m \in U)$, i.e. the genre assignments of a set of movies $U$ (note we consider $g(U)$ an ordered $|U|$-tuple, rather than a set). In this notation, we claim $\Pr[g(x_{t+1}) | (x_i, g(x_i)), i \in [1,t], x_{t+1}]$ is uniform over all $T$-tuples that don't include more than $s$ distinct genres (where the probability is taken over both the genre assignments of the movies, and the $t$ samples). By Bayes law, it suffices to show  
\begin{align*} &\Pr[(x_i, g(x_i)), i \in [1,t], x_{t+1} | g(x_{t+1}) = H] = \Pr[(x_i, g(x_i)), i \in [1,t], x_{t+1} | g(x_{t+1}) = H'] \end{align*} 
for all $H,H'$ that don't include more than $s$ distinct genres. We will massage this expression a bit, introducing the notation $I_t = ((x_i, g(x_i)), i \in [1,t])$. 
\begin{align*} &\Pr[I_t, x_{t+1} | g(x_{t+1}) = H] \\ 
&= \Pi_{i=1}^t \Pr[x_i, g(x_i)| g(x_{t+1}) = H, I_{t-1}] \\ 
&= \Pi_{i=1}^t \sum_{|h_i|_0 = s} \Pr\left[h_i| g(x_{t+1}) = H, I_{i-1}\right] \Pr\left[g(x_i) | h_i, g(x_{t+1}) = H, x_1, I_{i-1}\right] \Pr\left[x_i | g(x_i), h_i, g(x_{t+1}) = H, I_{i-1}\right]  \end{align*} 
On the other hand, since $h_i$ is independent of $g(x_{t+1})$ and $I_{i-1}$, $\Pr\left[h_i| g(x_{t+1}) = H, I_{i-1}\right] $ is uniform over all $h_i$, s.t. $|h_i|_0 = s$. Hence,   
\begin{align}
& \Pr[I_t, x_{t+1} | g(x_{t+1}) = H]  \nonumber\\
&\propto \Pi_{i=1}^t \sum_{g(x_i), |h_i|_0 = G/2} \Pr[g(x_i)| h_i, g(x_{t+1}) = H, I_{i-1}]  \Pr[x_i| g(x_i), h_i, g(x_{t+1}) = H, I_{i-1}]  \label{eq:temp1} \end{align}
Now, note that $\Pr[g(x_i)| h_i, g(x_{t+1}) = H, I_{i-1}] = 0$ if the genres in $g(x_i)$ are not included in $h_i$, and furthermore $ \Pr[x_i| g(x_i), h_i, g(x_{t+1}) = H, I_{i-1}] = 1/\binom{sM/k}{T}$, for those $g(x_i)$. 
From these two observations it follows that \eqref{eq:temp1} is independent of $H$, which finishes the proof of the theorem.  
  


\end{proof} 

\end{document}